\documentclass{article}
\usepackage{fullpage, authblk}
\usepackage{caption}
\usepackage{subcaption}
\usepackage[pdftex]{graphicx}

\usepackage[utf8]{inputenc}
\usepackage{amsmath,amssymb,amsfonts, amsthm} 
\usepackage[ruled, linesnumbered]{algorithm2e}

\title{Best of Both Worlds in Online Control: \\ Competitive Ratio and Policy Regret}

\author[1]{Gautam Goel\thanks{ggoel@berkeley.edu}}
\author[2]{Naman Agarwal\thanks{namanagarwal@google.com}}
\author[3]{Karan Singh\thanks{karansingh@cmu.edu}}
\author[2,4]{Elad Hazan\thanks{ehazan@princeton.edu}}
\affil[1]{Simons Institute, UC Berkeley}
\affil[2]{Google AI Princeton}
\affil[3]{Tepper School of Business, Carnegie Mellon University}
\affil[4]{Department of Computer Science, Princeton University}
\date{}

\usepackage{amsmath,amsfonts,amssymb}
\usepackage{mathtools}
\usepackage{amsthm} 
\usepackage{latexsym}
\usepackage{relsize}

 \DeclareMathOperator{\poly}{poly}
 \DeclareMathOperator{\polylog}{polylog}

\usepackage[capitalise]{cleveref}
\usepackage{xcolor}
\usepackage{dsfont}

\newcommand{\defeq}{\stackrel{\text{def}}{=}}

\newcommand{\K}{\ensuremath{\mathcal K}}

\newcommand{\ignore}[1]{}

\newcommand{\eh}[1]{\noindent{\textcolor{blue}{\{{\bf EH:} \em #1\}}}}

\theoremstyle{plain}
\newtheorem{theorem}{Theorem}
\newtheorem{lemma}[theorem]{Lemma}

\newtheorem{claim}[theorem]{Claim}

\newtheorem{assumption}{Assumption}

\newtheorem*{theorem*}{Theorem}
\newtheorem*{lemma*}{Lemma}
\newtheorem*{corollary*}{Corollary}
\newtheorem*{proposition*}{Proposition}
\newtheorem*{claim*}{Claim}
\newtheorem*{fact*}{Fact}
\newtheorem*{observation*}{Observation}
\newtheorem*{assumption*}{Assumption}

\theoremstyle{definition}
\newtheorem{definition}[theorem]{Definition}

\newtheorem*{definition*}{Definition}
\newtheorem*{remark*}{Remark}
\newtheorem*{example*}{Example}

 \theoremstyle{plain}
\newtheorem*{theoremaux}{\theoremauxref}
\gdef\theoremauxref{1}

\DeclareMathAlphabet{\mathbfsf}{\encodingdefault}{\sfdefault}{bx}{n}

\DeclareMathOperator*{\argmin}{arg\,min}

\newcommand{\reals}{\mathbb{R}}

\renewcommand{\leq}{~\le~}
\renewcommand{\geq}{~\ge~}

\let\oldtfrac\tfrac
\renewcommand{\tfrac}[2]{\smash{\oldtfrac{#1}{#2}}}

\let\nablaold\nabla
\renewcommand{\nabla}{\nablaold\mkern-2.5mu}

\begin{document}

\maketitle

\begin{abstract}
We consider the fundamental problem of online control of a linear dynamical system from two different viewpoints: regret minimization and competitive analysis. 
We prove that the optimal competitive policy is well-approximated by a convex parameterized policy class, known as a disturbance-action control (DAC) policies. Using this structural result, we show that several recently proposed online control algorithms achieve the best of both worlds: sublinear regret vs. the best DAC policy selected in hindsight, and optimal competitive ratio, up to an additive correction which grows sublinearly in the time horizon. We further conclude that sublinear regret vs. the optimal competitive policy is attainable when the linear dynamical system is unknown, and even when a stabilizing controller for the dynamics is not available \textit{a priori}. 
\end{abstract}

\ignore{
\section*{Outline of what we want to achieve}

\begin{enumerate}
    \item show that the optimal competitive policy can be approximated (up to additive regret terms) by a DAC on the original noise terms in the original dimensionality. 
    
    \item Research the above if true for $\ell_2$ vs. $\ell_\infty$ bounded noises. 
    
    \item Conclude via GPC a regret bound vs. the optimal competitive policy. (advantage of this algorithm - efficient gradient based)
    
    \item (conjectured/TBD for now) see if the regret bounds for unknown systems imply the same vs. the optimal competitive policies. i.e. adding adversarial sys id with method of moments.

Methodology to prove this:
\begin{enumerate}
    \item First we show that regret bound for LTI control --> regret vs. optimal competitive policy (independent of knowing the system) - existential
    
    \item Exists known efficient regret min vs. unknown system --> w. first point, what we want
\end{enumerate}

        \item Is it possible to get simultaneous bounded competitive ratio and sublinear regret? (lower bound?!)
\end{enumerate}
}

\section{Introduction}

The study of online optimization consists of two main research directions. The first is online learning, which studies regret minimization in games. A notable framework within this line of work is online convex optimization, where an online decision maker iteratively chooses a point in a convex set and receives loss according to an adversarially chosen loss function. The metric of performance studied in this research thrust is regret, or the difference between overall loss and that of the best decision in hindsight. 

The second direction is that of competitive analysis in metrical task systems. In this framework, the problem setting is similar, but the performance metric is very different. Instead of regret, the objective is to minimize the competitive ratio, i.e. the ratio of the reward of the online decision maker to that associated with the optimal sequence of decisions made in hindsight. For this ratio to remain bounded, an additional penalty is imposed on movement costs, or changes in the decision. 

While the goals of the two research directions are similar, the performance metrics are very different and lead to different algorithms and methodologies. These two separate methodologies have recently been applied to the challenging setting of online control, yielding novel and exciting methods to the field of control of dynamical systems. 

In the paper we unify these two disparate lines of work by establishing a connection between the two objectives. Namely, we show that the Gradient Perturbation Controller (GPC)  minimizes regret against the policy that has the {\it optimal competitive ratio} for any Linear Time Invariant (LTI) dynamical system.
The GPC algorithm hence gives the  best of both worlds: sublinear regret, and optimal competitive ratio, in a single efficient algorithm. 

Our main technical contribution is proving that the optimal competitive policy derived in \cite{goel2021competitive} is well-approximated by a certain convex policy class, for which efficient online learning was recently established in the work of \cite{agarwal2019online}. This implies that known regret minimization algorithms for online control can compete with this optimal competitive policy with vanishing regret. 

This structural result has other important implications to online control, yielding new results:  we show that sublinear regret can be attained vs. the optimal competitive policy even when the underlying dynamical system is unknown, and even when a stabilizing controller is not available.

\subsection{Related work}

{\bf Control of dynamical systems.} Our study focuses on two types of algorithms for online control. The first class of algorithms  enjoy sublinear regret for online control of dynamical systems; that is, whose performance tracks a given benchmark of policies up to a term which is vanishing relative to the problem horizon. \cite{abbasi2011regret} initiated the study of online control under the regret benchmark for linear time-invariant (LTI) dynamical systems. Bounds for this setting have since been improved and refined in \cite{dean2018regret,mania2019certainty,cohen2019learning,simchowitz2020naive}.
We are interested in adversarial noise and perturbations, and regret in the context of online control was initiated in the study of  \emph{nonstochastic} control setting \cite{agarwal2019online}, that allows for adversarially chosen (e.g. non-Gaussian) noise and general convex costs that may vary with time. This model has been studied for many extended settings, see \cite{hazan2021tutorial} for a comprehensive survey. 

{\bf Competitive control.} \cite{goel2019online} initiated the study of online control with competitive ratio guarantees and showed that the Online Balanced Descent algorithm introduced in \cite{chen2018smoothed} has bounded competitive ratio in a narrow class of linear systems. This approach to competitive control was extended in a series of papers \cite{goel2019beyond, shi2020online}.  In recent work, \cite{goel2021competitive} obtained an algorithm with optimal competitive ratio in general linear systems using $H_{\infty}$ techniques; in this paper we show that the competitive control algorithm obtained in \cite{goel2021competitive} is closely approximated by the class of DAC policies, and use this connection to obtain our ``best-of-both-worlds" result.

{\bf Online learning and online convex optimization.} The regret minimization techniques that are the subject of this paper are based in the framework of online convex optimization, see \cite{hazan2019introduction}. Recent techniques in online nonstochastic control are based on extensions of OCO to the setting of loss functions with memory \cite{anava2015online} and adaptive or dynamic regret \cite{hazan2009efficient,zhang2018dynamic}.

{\bf Competitive analysis of online algorithms and simultaneous bounds on competitive ratio and regret.} Competitive analysis was introduced in \cite{sleator1985amortized} and was first studied in the context of Metrical Task Systems (MTS) in \cite{borodin1992optimal}; we refer to \cite{borodin2005online} for an overview of competitive analysis and online algorithms.  A series of recent papers consider the problem of obtaining online algorithms with bounded competitive ratio and sublinear regret. In \cite{andrew2013tale}, it was shown no algorithm can simultaneously achieve both objectives in OCO with switching costs.  On the other hand, \cite{daniely2019competitive} described an online algorithm for MTS with optimal competitive ratio and sublinear regret on every time interval.

\section{Preliminaries}
We consider the task of online control in linear time-invariant (LTI) dynamical systems. In this setting, the interaction between the learner and the environment proceeds as described next. At each time step, the learner incrementally observes the current state $x_t \in \reals^m$ of the system, subsequently chooses a control input $u_t\in \reals^n$, and consequently is subject to an instantaneous cost $c(x_t,u_t)$ defined via the quadratic cost function (we assume the existence of $\beta, \mu$ such that $\beta I\succeq Q,R\succeq \mu I$)
$$c(x,u) = x^\top Q x + u^\top Ru.$$ 
As a consequence of executing the control input $u_t$, the dynamical system evolves to a subsequent state $x_{t+1}$, as dictated by the following linear system parameterized by the matrices $A\in \reals^{m\times m}$ and $B\in \reals^{m\times n}$, bounded as $\|A\|,\|B\|\leq \kappa$, and the perturbation sequence $(w_t)_{t\in [T]}$.
$$ x_{t+1} = Ax_t + B u_t + w_t. $$
We assume without loss of generality that $x_1 = 0$.
The learner does not directly observe the perturbations, or know of them in advance. We do not make any (e.g., distributional) assumptions on the perturbations, other than that they satisfy a point-wise bound $\|w_t\|\leq W$ for all times steps $t$.
By the means of such interaction across $T$ time steps, we ascribe an aggregate cost to the learner $\mathcal{A}$ as
$$ J_T(\mathcal{A}|w_{1:T}) = \sum_{t=1}^T c(x_t,u_t). $$


\subsection{Policy classes}
Since the learner selects the control inputs adaptively upon observing the state, the behavior of a learner may be described by a (strictly causal) policy $\pi$, a mapping from the observed state sequence to the immediate action.
We consider the following policy classes in the paper:
\begin{enumerate}
	\item $\Pi_{\textsc{all}}$ is the exhaustive set of $T$-length sequence of control inputs.
	\item $\Pi_{\textsc{SC}}$ is the class of all strictly causal policies, mapping the {\em heretofore} observed state sequence to the next action.
	\item $\mathcal{K} \subset \reals^{n\times m}$ is a class of linear state-feedback policies. Each member of this class is parameterized by some matrix $K \in \mathcal{K}$, and recommends the immediate action $u_t \defeq Kx_t$. Both the stochastic-optimal policy ($\mathcal{H}_2$-control) -- Bayes-optimal for i.i.d. perturbations -- and the robust policy ($\mathcal{H}_\infty$-control) -- minimax-optimal for arbitrary perturbations -- are linear state-feedback policies.
	\item $\mathcal{M}$ is a class of disturbance-action controllers (DAC), defined below, that recommend actions as a linear transformation of the past few perturbations, rather than the present state.
\end{enumerate}

A linear policy $K$ is called stable if $|\lambda_{\texttt{max}}(A+BK)|<1$. Such policies ensure that the state sequence remains bounded under their execution. The notion of strong stability, introduced by \cite{cohen2018online}, is a non-asymptotic characterization of the notion of stability defined as follows.
\begin{definition}
A linear policy $K\in \reals^{n\times m}$ is said to be $(\kappa,\gamma)$-strongly stable with respect to an LTI $(A,B)$ if there exists exists matrices $S,L$ satisfying $A+BK=SLS^{-1}$ such that 
$$\max\{1,\|K\|, \|S\|\|S^{-1}\|\}\leq \kappa
    \text{ and } \max\{1/2,\|L\|\}\leq 1-\gamma. $$
\end{definition}

A sufficient condition for the existence of a strongly stable policy is the strong controllability of the linear system $(A,B)$, a notion introduced in \cite{cohen2018online}. In words, strong controllability measures the minimum length and magnitude of control input needed to drive the system to any unit-sized state.


Let $\mathbb{K}$ be a fixed $(\kappa,\gamma)$-strongly stable linear policy for the discussion that follows. We will specify a particular choice for $\mathbb{K}$ in Section~\ref{sec:statements}. We formally define a disturbance action controller below. The purpose of superimposing a stable linear policy $\mathbb{K}$ on top of the linear-in-perturbation terms is to ensure that the state sequence produced under the execution of a (possibly non-stationary) disturbance-action controller remains bounded.

\begin{definition}
A disturbance-action controller (DAC), specified by a horizon $H$ and  parameters $M=\left(M^{[0]},\dots M^{[H-1]}\right) \in \reals^{n\times m}$, chooses the action at the time $t$ as 
$$ u_t(M) \defeq \mathbb{K}x_t + \sum_{t=1}^H M^{[i-1]}w_{t-i}, $$
where $x_t$ is state at time $t$, and $w_1,\dots w_{t-1}$ are past perturbations.
\end{definition}

\begin{definition}
For any $H\in\mathbb{N}, \gamma<1,\theta \geq 1$, an $(H,\theta,\gamma)$-DAC policy class is the set of all $H$-horizon DAC policies where $M=\left(M^{[0]},\dots M^{[H-1]}\right)$ satisfy $ \forall\;i \|M^{[i]}\|\leq \theta(1-\gamma)^i$.
\end{definition}

\subsection{Performance measures}
This paper considers multiple criteria that may be used to assess the learner's performance. We introduce these below. 

Let $w_{1:T}$ be the perturbation sequence the dynamics are subject to. Given the foreknowledge of this sequence, we define the following notions of optimal cost; note that these notions are {\em infeasible} in the sense that no online learner can match these on all instances. 
\begin{enumerate}
	\item $OPT_*(w_{1:T}) \defeq \min_{u_{1:T}\in \Pi_{\textsc{All}}} J_T(u_{1:T}|w_{1:T})$ is the cost associated with the best sequence of control inputs given the perturbation sequence. No policy, causal or otherwise, can attain a cost smaller than $OPT_*(w)$ on the the perturbation sequence $w_{1:T}$.
	\item For any policy class $\Pi$, $OPT_\Pi(w_{1:T}) \defeq \min_{\pi\in \Pi} J_T(\pi|w_{1:T})$ is the cost of the best policy in $\Pi$, subject to the perturbation sequence. Note that $OPT_*(w_{1:T}) = OPT_{\Pi_{\textsc{All}}}(w_{1:T})$.
\end{enumerate}

With respect to these baselines, we define the following performance measures.

{\bf Competitive Ratio (CR):} The competitive ratio of an (online) learner $\mathcal{A}$ is the worst-case ratio of its cost to the optimal offline cost $OPT_*(w_{1:T})$ over all possible perturbation sequences.
$$ \alpha_T(\mathcal{A}) \defeq \max_{w_{1:T}} \frac{J_T(\mathcal{A}|w_{1:T})}{OPT_*(w_{1:T})} $$ 
The optimal competitive ratio is the competitive ratio of the best strictly causal controller.
$$ \alpha^*_T \defeq \min_{\mathcal{A}\in \Pi_{\textsc{SC}}} \alpha_T(\mathcal{A}) $$ 
The infinite-horizon optimal competitive ratio $\alpha^*=\lim_{T\to\infty} \alpha^*_T$ is defined as the limiting optimal competitive ratio as the horizon extends to infinity, whenever it exists.

{\bf Regret:} On any perturbation sequence $w_{1:T}$, given a policy class $\Pi$, the regret of an online learner $\mathcal{A}$ is assigned to be the excess aggregate cost incurred in comparison to that of the best policy in $\Pi$, 
$$ R_{T,\Pi}(\mathcal{A}|w_{1:T}) = J_T(\mathcal{A}|w_{1:T}) - OPT_\Pi(w_{1:T}).$$
The worst-case regret is defined as the maximum regret attainable over all perturbation sequences,
$$ R_{T,\Pi}(\mathcal{A}) = \max_{w_{1:T}} R_{T,\Pi}(\mathcal{A}|w_{1:T}).$$
The two types of performance guarantees introduced above are qualitatively different in terms of the bound they espouse and the baseline they compare to. In particular: 

{\bf Tighter bound for low regret:} A sub-linear regret guarantee implies that the average costs of the learner and the baseline asymptotically match, while even an optimal competitive-ratio bound promises an average cost at most a constant factor {\em times} that of the baseline.

{\bf Stronger baseline for CR:} Competitive ratio bounds holds against the optimal unrestricted policy while regret holds against the best fixed policy from a (typically parametric) policy class.

\subsection{Characterization of the optimal CR algorithm}
The following explicit characterization of a strictly causal policy that achieves an optimal competitive ratio in the infinite-horizon setting was recently obtained in \cite{goel2021competitive}; this theorem shows that the competitive policy in the original system with state $x \in \mathbb{R}^m$ can be viewed as a state-feedback controller in a synthetic system with state $\xi \in \mathbb{R}^{2m}$.

\begin{theorem}[Optimal Competitive Policy]\label{crt}
The strictly causal controller with an optimal infinite-horizon competitive ratio $\alpha^*$ is given by the policy $u_t = \widehat{K} \xi_t$, where $\widehat{K} \in \reals^{n \times 2m}$ and the synthetic state $\xi \in \reals^{2m}$ evolves according to the dynamics 
$$\xi_{t+1} = \widehat{A} \xi_t + \widehat{B}_u u_t + \widehat{B}_w \hat{w}_{t+1}, $$
 $$\text{where }\widehat{A} = \begin{bmatrix} A & K\Sigma^{1/2} \\ 0 & 0 \end{bmatrix}, \quad \widehat{B}_u = \begin{bmatrix} B \\ 0 \end{bmatrix}, \quad \widehat{B}_w = \begin{bmatrix} 0 \\ I \end{bmatrix}, \quad \widehat{w}_t = \Sigma^{-1/2}Q^{1/2}\nu_t. $$
The sequence $\nu_t$ is recursively defined as $ \nu_{t+1} = (A-KQ^{1/2})\nu_t + w_t$ starting with $\nu_1=0$. Here the matrices $\widehat{K}, K, \Sigma$ (and auxiliary constants $P,\widetilde{B}, \widetilde{H}$ and $\widehat{P}$) satisfy
$$ K = APQ^{1/2}\Sigma^{-1}, \quad \Sigma = I + Q^{1/2} P Q^{1/2}, \quad P = BB^\top + APA^\top - K\Sigma K^\top, $$
$$ \widehat{K} = -(I_n+\widehat{B}_u^\top \widetilde{P}\widehat{B}_u)^{-1} \widehat{B}_u^\top \widetilde{P} \widehat{A},  \quad \widetilde{B}= \begin{bmatrix}\widehat{B}_u &  \widehat{B}_w\end{bmatrix}, \quad \widetilde{H} = \begin{bmatrix}
	I & 0 \\
	0& -\alpha^* I
\end{bmatrix} + \widetilde{B}^\top \widehat{P} \widetilde{B},$$
$$\widetilde{P} = \widehat{P} - \widehat{P}\widehat{B}_w(-\alpha^* I_p + \widehat{B}_w^\top\widehat{P}\widehat{B}_w)^{-1}\widehat{B}_w^\top\widehat{P}, $$
$$ \text{and } \quad \widehat{P} = \begin{bmatrix}
	Q & Q^{1/2}\Sigma^{1/2}, \\
	Q^{1/2}\Sigma^{1/2} & \Sigma
\end{bmatrix} + \widehat{A}^\top \widehat{P}\widehat{A} - \widehat{A}^\top \widehat{P} \widetilde{B} \widetilde{H}^{-1} \widetilde{B}^\top \widehat{P}\widehat{A} .$$
Furthermore, let $\{x_t\}_{t=1}^T$ be the state sequence produced under the execution of such a policy. Then, the state sequence satisfies at all time $t$ that
$ \xi_t = \begin{bmatrix} x_t-\nu_t \\
\widehat{w}_t\end{bmatrix}.$
\end{theorem}

Let $\widehat{K}_0\in \reals^{n\times m}$ be the sub-matrix induced by the first $m$ columns of $\widehat{K}$. In general, the infinite-horizon optimal competitive ratio may not be finite. However, the stability of the associated filtering operation (i.e. $|\lambda_{\max} (A-KQ^{1/2})|<1$) and the closed loop control system (i.e. $|\lambda_{\max}(A+B\widehat{K}_0)|<1$) is sufficient to ensure the existence of this limit. We utilize the following bounds that quantify this. 

\begin{assumption}
	$\widehat{K}_0$ is $(\kappa,\gamma)$-strongly stable with respect to the linear system $(A,B)$, and $-K^\top$ is $(\kappa,\gamma)$-strongly stable with respect to the linear system $(A^\top, Q^{1/2})$. Also, $\|\widehat{K}\| \leq \kappa$.
\end{assumption}

We note that the above bounds are quantifications, and not strengthening, of the stability criterion. In particular, any stable controller is strongly stable for some $\kappa\geq 1, \gamma<1$. Here, we use the same parameters to state the strong stability for both controllers, $K$ and $\widehat{K}$, for convenience. Such a simplification is valid, since given $(\kappa_1,\gamma_1)$- and $(\kappa_2,\gamma_2)$-strongly stable controllers, the said controllers are also $(\max\{\kappa_1,\kappa_2\},\max\{\gamma_1,\gamma_2\})$-strongly stable. 

\subsection{Low-regret algorithms}

Deviating from the methodologies of optimal and robust control,
\cite{agarwal2019online} propose considering an online control formulation in which the noise is adversarial, and thus the optimal controller is only defined in hindsight. This motivates different, online-learning based methods for the control task. \cite{agarwal2019online} proposed an algorithm called GPC (Gradient Perturbation Controller) and show the following theorem (which we restate in notation consistent with this paper), which for LTI systems shows that regret when compared against any strongly-stable policy scales at most as $O(\sqrt{T})$. 

\begin{theorem}
Given any $\kappa, \gamma$, let $\mathcal{K}(\kappa, \gamma)$ be the set of $(\kappa,\gamma)$ strongly stable linear policies. There exists an algorithm $\mathcal{A}$ such that the following holds,
\[ J_{T}(\mathcal{A}|w_{1:T}) - \min_{K \in \mathcal{K}(\kappa, \gamma)} J_T(K | w_{1:T}) \leq O(\sqrt{T}\log(T)).\]
Here $O(\cdot)$ contains polynomial factors depending on the system constants.  
\end{theorem}

As can be observed from the analysis presented by \cite{agarwal2019online}, the above regret guarantee holds not just against the set of strongly stable linear policies but also against the set of $(H, \theta, \gamma)$-DAC policies. The regret bound has been extended to different interaction models such as unknown systems \cite{pmlr-v117-hazan20a}, partial observation \cite{simchowitz2020improper} and adaptive regret \cite{gradu2020adaptive}. Furthermore the regret bound in this setting has been improved to logarithmic in $T$ \cite{pmlr-v119-foster20b, simchowitzmaking}. 

\section{Statement of results}\label{sec:statements}

\subsection{Main Result}
The central observation we make is that regret-minimizing algorithms subject to certain qualifications automatically achieve an optimal competitive ratio bound up to a vanishing average cost term. 

Typical regret-minimizing online control algorithms \cite{agarwal2019online,hazan2019nonstochastic} compete against the class of stable linear state-feedback policies. In general, neither the offline optimal policy (with cost $OPT_*$)  nor the optimal competitive-ratio policy can be approximated by a linear policy \cite{goel2020power}. However, the algorithm proposed in \cite{agarwal2019online} and follow-up works that build on it also compete with a more-expressive class, that of disturbance-action policies (DACs). In \cite{agarwal2019online}, this choice was made purely for computational reasons to circumvent the non-convexity of the cost associated with linear policies; in this work, however, we use the flexibility of DAC policies to approximate the optimal competitive policy.

More formally, we prove that we can find a DAC which generates a sequence of states and control actions which closely track the sequence of states and control actions generated by the optimal  competitive policy by taking the history $H$ of the DAC to be sufficiently large. This structural characterization of the competitive policy is sufficient to derive our best-of-both-worlds result, since a regret-minimizing learner competitive against an appropriately defined DAC class would also be competitive against the policy achieving an optimal competitive ratio, and hence achieve an optimal competitive ratio up to a residual regret term. 

\begin{theorem}[Optimal Competitive Policy is Approximately DAC]\label{theorem:main}
Fix a horizon $T$ and a disturbance bound $W$. For any $\varepsilon > 0$, set
$$H =  \log(1-\gamma/2)^{-1}\log\left(\frac{1088W^2\kappa^{11}\max(1, \beta^2)}{\gamma^4 \varepsilon}T \right), \hspace{1cm} \theta = 2\kappa^2\max(1, \beta^{1/2})$$
and define $\mathcal{M}$ be the set of $(H,\theta,\gamma)$-DAC policies with stabilizing component $\mathbb{K}=\widehat{K}_0$. Let $\mathcal{A}$ be the algorithm with the optimal competitive ratio $\alpha^*$. Then there exists a policy $\pi\in \mathcal{M}$ such that for any perturbations $w_{1:T}$ satisfying $\|w_t\| \leq W$, the cost incurred by $\pi$ satisfies
$$J_T(\pi|w_{1:T}) <  J_T(\mathcal{A}|w_{1:T}) + \varepsilon.$$
\end{theorem}
We prove this theorem in Section \ref{sec:structural}. We now show that this result implies best-of-both-worlds.


\subsection{Best-of-Both-Worlds in Known Systems}
 We begin by considering the case when the learner knows the linear system $(A,B)$. In this setting, both the regret and competitive ratio thus measure the additional cost imposed by not knowing the perturbations in advance. The result below utilizes the regret bounds against DAC policies and associated algorithms from \cite{agarwal2019online,simchowitzmaking}.

\begin{theorem}[Best-of-both-worlds in online control (known system)]\label{thm:main}
Assuming $(A,B)$ is known to the learner, there exists a constant 
$$ R_T = \tilde{O}\left(\poly(m,n,\beta,\kappa,\gamma^{-1})W^2\times \min\left\{\sqrt{T}, \poly(\mu^{-1})\polylog T\right\}\right) $$
and a computationally efficient online control algorithm $\mathcal{A}$ which simultaneously achieves the following performance guarantees: 
\begin{enumerate}
    \item \label{optimal-competitive-ratio-property} (Optimal competitive ratio) The cost of $\mathcal{A}$ satisfies for any perturbation sequence $w_{1:T}$ that $$J_T(\mathcal{A}|w_{1:T}) < \alpha^* \cdot OPT_* (w_{1:T}) + R_T,$$
    where $\alpha^*$ is the optimal competitive ratio.
    \item \label{low-regret-property} (Low regret) The regret of $\mathcal{A}$ relative to the best linear state-feedback or DAC policy selected in hindsight grows sub-linearly in the horizon $T$, i.e. for all $w_{1:T}$, it holds 
    $$J_T(\mathcal{A}|w_{1:T}) < \min_{\pi \in \K} J_T(\pi|w_{1:T})+ R_T \quad \text{ and }\quad J_T(\mathcal{A}|w_{1:T}) < \min_{\pi \in \mathcal{M}} J_T(\pi|w_{1:T})+ R_T.$$
\end{enumerate}
\end{theorem}

\subsection{Best-of-Both-Worlds in Unknown Systems}
We now present the main results for online control of unknown linear dynamical system. The first theorem deals with the case when the learner has coarse-grained information about the linear system $(A,B)$ in the form of access to a stabilizing controller $\mathbb{K}$. In general, to compute such a stable controller, it is sufficient to known $(A,B)$ to some constant accuracy, as noted in \cite{cohen2019learning}. This theorem utilizes low-regret algorithms from \cite{hazan2019nonstochastic,simchowitzmaking}.

\begin{theorem}[Best-of-both-worlds in online control (unknown system, given a stabilizing controller)]\label{thm:give}
For a $(k,\kappa)$-strongly controllable linear dynamical system $(A,B)$, there exists a constant 
$$ R_T = \tilde{O}\left(\poly(m,n,\beta,k,\kappa,\gamma^{-1})W^2\times \min\left\{T^{2/3}, \poly(\mu^{-1})\sqrt{T}\right\}\right) $$
and a computationally efficient online control algorithm $\mathcal{A}$ such that, when given access to a $(\kappa,\gamma)$-strongly stable initial controller $\mathbb{K}$, it guarantees
    $$ J_T(\mathcal{A}|w_{1:T}) < \min\{\alpha^*\cdot OPT_*(w_{1:T}), \min_{\pi \in \K} J_T(\pi|w_{1:T}), \min_{\pi \in \mathcal{M}} J_T(\pi|w_{1:T})\}+ R_T.$$
\end{theorem}

When an initial stabilizing controller is unavailable, we make use of the ``blackbox control'' algorithm in \cite{chen} to establish the next theorem. 

\begin{theorem}[Best-of-both-worlds in online control (unknown system, blackbox control)]\label{thm:black}
For a $(k,\kappa)$-strongly controllable linear dynamical system $(A,B)$,
there exists a constant 
$$ R_T = 2^{\poly(m,n,\beta, k,\kappa,\gamma^{-1})}+ \tilde{O}\left(\poly(m,n,\beta,k,\kappa,\gamma^{-1})W^2\times \min\left\{T^{2/3}, \poly(\mu^{-1})\sqrt{T}\right\}\right) $$
and a computationally efficient online control algorithm $\mathcal{A}$ that guarantees
    $$ J_T(\mathcal{A}|w_{1:T}) < \min\{\alpha^*\cdot OPT_*(w_{1:T}), \min_{\pi \in \K} J_T(\pi|w_{1:T}), \min_{\pi \in \mathcal{M}} J_T(\pi|w_{1:T})\}+ R_T.$$
\end{theorem}

\section{Proof of Theorem \ref{theorem:main}}\label{sec:structural}
\begin{proof}[Proof of Theorem \ref{theorem:main}] 
Unrolling the recursive definition of $\nu_t$ in Theorem~\ref{crt}, we see that
$$ \nu_t = \sum_{i = 0}^{t-1} (A-KQ^{1/2})^{i-1} w_{t-i}.$$ Let $\{x_t\}_{t=1}^T$ be the state sequence generated by the competitive control policy in the original $m$-dimensional system and let $\{\xi_t\}_{t=1}^T$ is the sequence of states generated by the competitive policy in the $2m$-dimensional synthetic system. Additionally, let $\widehat{K} = \begin{bmatrix} \widehat{K}_0 & \widehat{K}_1 \end{bmatrix}$ be the partition of $\widehat{K}$ along the first $m$ columns and the remaining $m$. Recall that  $A + B\widehat{K}_0 = SLS^{-1}$, where $\|S\|\|S^{-1}\| \leq \kappa$ and $\|L\| \leq 1 - \gamma$. Using Theorem~\ref{crt}, we decompose $x_t$ as a linear combination of the disturbance terms:
\begin{align*}
    x_{t+1} &= Ax_t+B\widehat{K}\xi_t +w_t = Ax_t + B\begin{bmatrix} \widehat{K}_0 & \widehat{K}_1 \end{bmatrix}\begin{bmatrix} x_t-\nu_t \\ \Sigma^{-1/2} Q^{1/2}\nu_t \end{bmatrix} + w_t\\
    & = (A+B\widehat{K}_0)x_t + B(\widehat{K_1} \Sigma^{-1/2}Q^{1/2} - \widehat{K}_0) \nu_t  + w_t\\
    &= \sum_{i= 0}^{t-1} \left((A+B\widehat{K}_0)^i  + \sum_{j=1}^{i} (A+B\widehat{K}_0)^{i-j} B(\widehat{K}_1\Sigma^{-1/2}Q^{1/2} -\widehat{K}_0)  (A-KQ^{1/2})^{j-1}  \right) w_{t-i}.
\end{align*}
We now describe a DAC policy which approximates the competitive control policy; recall that every DAC policy is parameterized by a stabilizing controller and a set of $H$ weights. We take the stabilizing controller to be $\widehat{K}_0$ and the set of weights to be $M = (M^{[0]}, \ldots, M^{[H-1]})$, where $M^{[i-1]}$ is
$$M^{[i-1]} \defeq  (\widehat{K}_1\Sigma^{-1/2}Q^{1/2}-\widehat{K}_0)(A-KQ^{1/2})^{i-1}.$$
The action generated by our DAC approximation of the competitive policy is therefore
\begin{align} \label{u-dac-dec}
u_t(M) \defeq \widehat{K}_0x_t(M) + \sum_{i=1}^H M^{[i-1]} w_{t-i},
\end{align}
where $x_t(M)$ is the corresponding state sequence generated by the DAC.

We now show that the weights $M^{[i-1]}$'s decay geometrically in time. Recall that  $A + B\widehat{K}_0 = S_1L_1S_1^{-1}$, where $\|S_1\|\|S_1^{-1}\| \leq \kappa$ and $\|L\| \leq 1 - \gamma$. Similarly, $A - KQ^{1/2} = S_2L_2S_2^{-1}$, where $\|S_2\|\|S_2^{-1}\| \leq \kappa$ and $\|L_2\| \leq 1 - \gamma$. Recall that $\Sigma = I + Q^{1/2}PQ^{1/2} \succeq I$, therefore $\|\Sigma^{-1/2}\| \leq 1$. We see that
\begin{align}\label{eq:mbound}
\|M^{[i-1]}\| \leq \|\widehat{K}_1\Sigma^{-1/2}Q^{1/2}-\widehat{K}_0\| \|(A-KQ^{1/2})^{i-1} \| \leq 2\kappa^2 \max(1, \|Q\|^{1/2}) (1-\gamma)^{i-1},
\end{align}
where we used the fact that $\kappa \geq \|\widehat{K}\|$ and hence $\kappa \geq \max(\|\widehat{K}_0\|, \|\widehat{K}_1\|)$. We now bound the distance between $x_t$ and $x_t(M)$ . Plugging our choice of $u_t$ given in (\ref{u-dac-dec}) into the dynamics of $x$, we see that
\begin{align*}
    x_{t+1}(M) &= Ax_{t}(M) + B \left(\widehat{K}_0x_{t}(M) + \sum_{i=1}^H M^{[i-1]} w_{t-i}\right) +  w_{t} \\
    &= \sum_{i = 0}^{t-1} \left((A+B\widehat{K}_0)^{i} + \sum_{j=1}^{\min\{H, i\}}  (A+B\widehat{K}_0)^{i-j} B M^{[j-1]} \right)w_{t-i}.
\end{align*}

Comparing the expansions of $x_t$ and $x_t(M)$ on a per-term basis, we observe that the terms in their difference contain either $(A+B\widehat{K}_0)^{H/2}$ or $(A-KQ^{1/2})^{H/2}$, and thus their difference is exponentially small in $H/2$ under the strong stability of $\widehat{K}_0$ and $K$. Formally, we have
\begin{claim}
\label{claim:xtxtM}
$\|x_{t+1} - x_{t+1}(M) \| \leq  \frac{16W\kappa^4\max(1, \|Q\|^{1/2})}{\gamma^2} (1 - \gamma/2)^H.$
\end{claim}
We now bound the distance between $u_t$ and $u_t(M)$. Unrolling the dynamics, we get the following, 
\begin{align*}
    u_t &= \widehat{K}\xi_t = \widehat{K}_0 x_t + \sum_{i = 0}^{t-1} (\widehat{K}_1\Sigma^{-1/2}Q^{1/2} -\widehat{K}_0)(A-KQ^{1/2})^{i-1} w_{t-i},
\end{align*}
$$u_t(M) = \widehat{K}_0x_t(M) + \sum_{i=1}^H (\widehat{K}_1\Sigma^{-1/2}Q^{1/2} -\widehat{K}_0 )(A-KQ^{1/2})^{i-1} w_{t-i}.$$
Similar to the argument for Claim \ref{claim:xtxtM} we have the following, 
\begin{claim}\label{claim:deltau}
$\|u_t-u_t(M)\| \leq \frac{20W\kappa^5\max(1, \|Q\|^{1/2})}{\gamma^2} (1 - \gamma/2)^H$.
\end{claim}
where we used the strong stability of $K^\top$, the bound on $\|x_t - x_t(M)\|$ given in Claim \ref{claim:xtxtM} and the fact that $\kappa \geq \max(1, \|\widehat{K}_0\|, \|\widehat{K}_1\|)$. We now show that by taking $H$ to be sufficiently large we can ensure that $J_T(\pi|w_{1:T}) <  J_T(\mathcal{A}|w_{1:T}) + \varepsilon.$
We utilize the following bound on $x_t(M), u_t(M)$:
\begin{lemma}\label{lem:xudacboundnwq}
The execution of any policy from a $(H,\theta,\gamma')$-DAC policy class produces state-action sequences that obey for all time $t$ that $\|x_t\|,\|u_t\|\leq {3\kappa^3\theta W}/{\gamma\gamma'}$.
\end{lemma}

In light of (\ref{eq:mbound}), we see that we can take $\theta = 2\kappa^2\max(1, \|Q\|^{1/2})$, leading to the bound
$$\max(\|x_t\|,\|u_t\|) \leq \frac{6W\kappa^5\max(1, \|Q\|^{1/2})}{\gamma^2}.$$

Using the easily-verified inequality $|\|x\|^2 - \|y\|^2|  \leq \left(\|x - y\| \right) \left(\|2x\| + \|x - y\| \right)$, we put the pieces together to bound the difference in costs incurred by the competitive policy and our DAC approximation in each timestep:
\begin{align*}
|x_t^\top Q x_t - x_t(M)^\top Q x^\top(M) |\leq& \|Q\| \|x_t-x_t(M)\|\left(2\|x_t(M)\|+ \|x_t-x_t(M)\|\right)\\
\leq &  \frac{448W^2\kappa^{10}\max(1, \|Q\|)\|Q\|}{\gamma^4}(1 - \gamma/2)^{H}
\end{align*}
Similarly,
\begin{align*}
|u_t^\top Ru_t - u_t(M)^\top R u_t(M)|\leq& \|R\| \|u_t-u_t(M)\|\left(2\|u_t(M)\|+ \|u_t-u_t(M)\|\right)\\
\leq &  \frac{640W^2\kappa^{11}\max(1, \|Q\|)\|R\|}{\gamma^4}(1 - \gamma/2)^{H}.
\end{align*}
The difference in aggregate cost is 
therefore
$$|J_T(\pi|w_{1:T}) - J_T(\mathcal{A}|w_{1:T})| \leq \frac{1088W^2\kappa^{11}\max(1, \beta^2)}{\gamma^4}(1 - \gamma/2)^{H}T.$$
Taking $ H\geq  \log\left({1088W^2\kappa^{11}\max(1,\beta^2)T}/{\gamma^4 \varepsilon} \right)/{\log(1-\gamma/2)}$ concludes the proof.
\end{proof}

\section{Experiments}

In Figure \ref{fig:dint} we compare the performance of various controllers, namely, the $H_2$ controller, the $H_{\infty}$ controller, the infinite horizon competitive controller from \cite{goel2021competitive}, the GPC controller from \cite{agarwal2019online} and the  ``clairvoyant offline'' controller which selects the optimal-in-hindsight sequence of controls. We do this comparison on a two dimensional double integrator system with different noise sequences. We confirm as the results of our paper suggest that the GPC controller attaining the best of both worlds guarantee is indeed the best performing controller and in particular matches and sometimes improves over the performance of competitive control. Further experiment details along with more simulations on different systems can be found in the appendix. 

\begin{figure}
     \centering
     \begin{subfigure}[b]{0.45\textwidth}
         \centering
         \includegraphics[width=\textwidth]{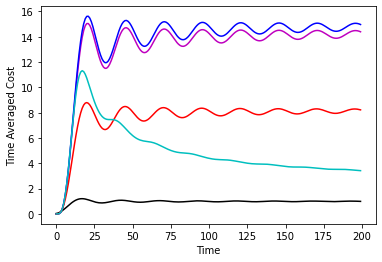}
         \caption{Sinusoidal Perturbations}
         \label{fig:dint_sin}
     \end{subfigure}
     \hfill
     \begin{subfigure}[b]{0.45\textwidth}
         \centering
         \includegraphics[width=\textwidth]{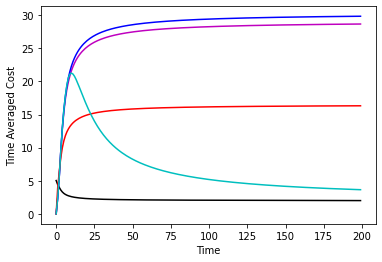}
         \caption{Constant Perturbations}
         \label{fig:dint_const}
     \end{subfigure}
     \bigskip 
     \begin{subfigure}[b]{0.45\textwidth}
         \centering
         \includegraphics[width=\textwidth]{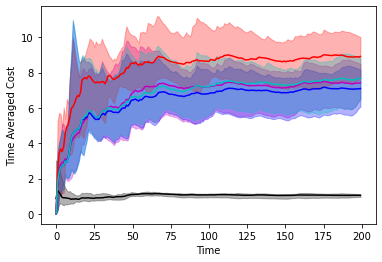}
         \caption{Gaussian Perturbations}
         \label{fig:dint_gaussian}
     \end{subfigure}
     \hfill
     \begin{subfigure}[b]{0.45\textwidth}
         \centering
         \includegraphics[width=\textwidth]{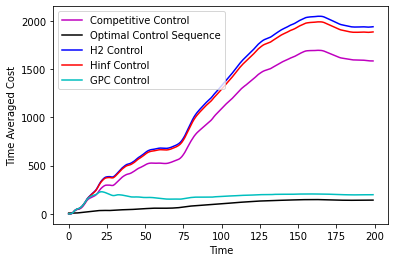}
         \caption{Gaussian Random Walk Perturbations}
         \label{fig:dint_grw}
     \end{subfigure}
     
        \caption{Relative performance of the linear-quadratic controllers  in the double integrator system.}
        \label{fig:dint}
        
\end{figure}

\section{Conclusions, Open Problems and Limitations} \label{sec:conclusions}

We have proved that the optimal competitive policy in an LTI dynamical system is well-approximated by the class of Disturbance Action Control (DAC) policies. This implies that the Gradient Perturbation Control (GPC) algorithm and related approaches are able to attain sublinear regret vs. this policy, even when the dynamical system is unknown ahead of time.  This is the first time that a control method is shown to attain both sublinear regret vs. a large policy class, and simultaneously a competitive ratio vs. the optimal dynamic policy in hindsight (up to a vanishing additive term). 
It remains open to extend our results to time varying and nonlinear systems, the recent methods of \cite{minasyan2021online,gradu2020adaptive} are a potentially good starting point.

\section*{Acknowledgements}

Elad Hazan gratefully acknowledges funding from NSF grant \#1704860.

\bibliography{refs}
\bibliographystyle{alpha}

\newpage

\appendix
 \section{Appendix Layout}
 
 In Section~\ref{sec:a}, we provide the proofs of the claims necessary to conclude Theorem~\ref{theorem:main}. Following this, Section~\ref{sec:b} \& \ref{sec:c} proves the results for known and unknown systems respectively. Thereafter, we prove a generalization of Theorem~\ref{theorem:main} necessary for the case of unknown systems in Section~\ref{sec:d}. Finally, Section~\ref{sec:e} discusses the translation of the optimality results from infinite-horizon to finite-horizon. Section~\ref{sec:exp} provides further details on the experimental setup, along with more experimental evaluations.

\section{Proofs of claims substantiating Theorem~\ref{theorem:main}}\label{sec:a}
\begin{proof}[Proof of Claim~\ref{claim:xtxtM}]
Using the strong stability of $\widehat{K}_0$ and $K^\top$, we have
 \begin{align}\label{eq:xbound}
     &\|x_{t+1} - x_{t+1}(M) \| \\
     \leq& \sum_{i\geq H+1}\sum_{j=H+1}^i   \left\|(A+B\widehat{K}_0)^{i-j}B (\widehat{K}_1\Sigma^{-1/2}Q^{1/2}-\widehat{K}_0)(A-KQ^{1/2})^{j-1} \right\|W \notag \\
     \leq&  2W\kappa^4\max(1, \|Q\|^{1/2})\sum_{i\geq H+1}i(1-\gamma)^{i-1}\notag\\
     \leq &  \frac{8W\kappa^4\max(1, \|Q\|^{1/2})}{\gamma} \sum_{i\geq H+1}  (1-\gamma/2)^i, \notag \\
     \leq & \frac{16W\kappa^4\max(1, \|Q\|^{1/2})}{\gamma^2} (1 - \gamma/2)^H
 \end{align}
 where in the penultimate line we used the following elementary result:
 \begin{lemma}\label{lem:newlemon}
 For all $\gamma \in [0, 1/2]$ and all $i \geq 0$, the following inequality holds: $$ i(1-\gamma)^{i} \leq 2(1-\gamma/2)^i/\gamma.$$
 \end{lemma}
\end{proof}

\begin{proof}[Proof of Claim~\ref{claim:deltau}] Recall that by unrolling the dynamics, we get the following, 
\begin{align*}
    u_t &= \widehat{K}\xi_t = \widehat{K}_0 (x_t-\nu_t)+\widehat{K}_1\widehat{w}_t= \widehat{K}_0 x_t + \sum_{i = 0}^{t-1} (\widehat{K}_1\Sigma^{-1/2}Q^{1/2} -\widehat{K}_0)(A-KQ^{1/2})^{i-1} w_{t-i},
\end{align*}
$$u_t(M) = \widehat{K}_0x_t(M) + \sum_{i=1}^H (\widehat{K}_1\Sigma^{-1/2}Q^{1/2} -\widehat{K}_0 )(A-KQ^{1/2})^{i-1} w_{t-i}.$$
Using the strong stability of $K^\top$, we bound the point-wise difference between these sequences as 
\begin{align*}
\|u_t-u_t(M)\| &\leq \|\widehat{K}_0\| \|x_t - x_t(M)\| + \sum_{i\geq H+1} \|(\widehat{K}_1\Sigma^{-1/2}Q^{1/2} -\widehat{K}_0 )(A-KQ^{1/2})^{i-1} \|W\\
&\leq \frac{20W\kappa^5\max(1, \|Q\|^{1/2}))}{\gamma^2} (1 - \gamma/2)^H.
\end{align*}
\end{proof}

\begin{proof}[Proof of Lemma~\ref{lem:xudacboundnwq}]
	Since $\mathbb{K}$ is $(\kappa,\gamma)$-strongly stable, there exists matrices $Q,L$ such that $QLQ^{-1}=A+B\mathbb{K}$ with $\|Q^{-1}\|\|Q\|\leq \kappa$ and $\|L\|\leq 1-\gamma$. Also recall that for every member of a $(H,\theta,\gamma')$-DAC class satisfies $\|M^{[i]}\|\leq \theta(1-\gamma')^{i}$. Observe for any $t\geq 0$ that 
	\begin{align*}
		\|x_{t+1}\| &= \left\|\sum_{i\geq 0} (A+B\mathbb{K})^i \left(w_{t-i} + B\sum_{j=1}^H M^{[j-1]}w_{t-i-j}\right)\right\|\\
		&\leq \left(\sum_{i\geq 0} \|(A+B\mathbb{K})^i\|\right) \max_i\left\|w_{t-i} + B\sum_{j=1}^H M^{[j-1]}w_{t-i-j}\right\|\\
		&\leq \left(\sum_{i\geq 0} \|(A+B\mathbb{K})^i\|\right) \max_i\left( \|w_{t-i}\| + \|B\| \max_j\|w_{t-i-j}\|\sum_{j=1}^H \|M^{[j-1]}\|\right)\\
		&\leq \frac{\kappa}{\gamma}W\left(1+\frac{\kappa \theta}{\gamma'}\right) \leq \frac{2\kappa^2 \theta W}{\gamma\gamma'},
	\end{align*}
	\begin{align*}
		\|u_t\| & = \left\|\mathbb{K}x_t + \sum_{i=1}^H M^{[i-1]}w_{t-i}\right\| \leq \frac{2\kappa^3 \theta W}{\gamma\gamma'} +\frac{\theta W}{\gamma'}.
	\end{align*} 
\end{proof}

 \begin{proof}[Proof of Lemma~\ref{lem:newlemon}]
Recall that $xe^{-x} \leq 1$ for all $x$, therefore $ie^{-\gamma i/2} \leq \frac{2}{\gamma}.$ Also $\sqrt{1 - \gamma} \leq 1 - \gamma/2$ for all $\gamma\leq 1$. We have 
\begin{eqnarray*}
i(1-\gamma)^{i} &=& i(1-\gamma)^{i/2}(1-\gamma)^{i/2}\\
&\leq&ie^{-\gamma i/2}(1-\gamma/2)^{i}\\
&\leq&\frac{2}{\gamma}(1-\gamma/2)^{i}.
\end{eqnarray*}
\end{proof}

\section{Proofs for known systems}\label{sec:b}
\begin{proof}[Proof of Theorem~\ref{thm:main}]
The perturbation sequence $w_{1:T}$ has $T$ terms. Overloading this notation, we extend this to an infinite sequence with $w_t=0$ for all $t>T$. Let $\K$ be the class of $(\kappa,\gamma)$-strongly stable linear policies, and $\mathcal{M}$ be a $(H,\theta,\gamma')$-DAC class for $$\theta=2\kappa^2\max\{1,\beta^{1/2}\}, \gamma' = \gamma/2, H=2\log(1088W^2\kappa^{11}\max\{1,\beta^2\}T/\gamma^4)/\gamma.$$
By the results of \cite{agarwal2019online}, using the GPC algorithm (Algorithm 1 in \cite{agarwal2019online}) as the online learner $\mathcal{L}$, we have the following regret guarantee:
$$ J_T(\mathcal{L}|w_{1:T}) \leq \min_{\pi\in \K\cup \mathcal{M}} J_T(\pi|w_{1:T}) + \tilde{O}(\poly(m,n,\kappa,\gamma^{-1},\beta)W^2\sqrt{T})$$
Now, recall that Theorem~\ref{theorem:main} asserts
$$ \min_{\pi\in \mathcal{M}}J_T(\pi| w_{1:T}) \leq J_T(\mathcal{A}|w_{1:T}) + O(1),  $$
where $\mathcal{A}$ is the online learner (from Theorem~\ref{crt}) that achieves an optimal competitive ratio $\alpha^*$.
By  the non-negativity of the cost, and thereafter using the optimal competitive ratio bound, we have
\begin{align*}
  J_T(\mathcal{\mathcal{L}} |w_{1:T}) &\leq  \min_{\pi\in \mathcal{M}} J_T(\pi | w_{1:T}) + \tilde{O}(\poly(m,n,\kappa,\gamma^{-1},\beta)W^2\sqrt{T}) \\
  &\leq J_T(\mathcal{A}|w_{1:T}) + \tilde{O}(\poly(m,n,\kappa,\gamma^{-1},\beta)W^2\sqrt{T})\\
  &\leq J_\infty(\mathcal{A}|w_{1:\infty}) + \tilde{O}(\poly(m,n,\kappa,\gamma^{-1},\beta)W^2\sqrt{T})\\
  &\leq \alpha^*\cdot OPT_*(w_{1:\infty}) + \tilde{O}(\poly(m,n,\kappa,\gamma^{-1},\beta)W^2\sqrt{T})
\end{align*}
An appeal to Theorem~\ref{mainmagic} concludes the $\sqrt{T}$-part of the claim. For the polylog regret result, we use the algorithm and the associated $O(\poly(\mu^{-1}\log T))$ regret guarantee in \cite{simchowitzmaking}.
\end{proof}

\section{Proofs for unknown systems}\label{sec:c}
\begin{proof}[Proof of Theorem~\ref{thm:give}]
The proof of this theorem is similar in structure to that of Theorem~\ref{thm:main}. In particular, we utilize regret-minimizing algorithms from \cite{hazan2019nonstochastic} \& \cite{simchowitzmaking} that apply even $A,B$ are unknown, given a $(\kappa,\gamma)$-strongly stable controller $\mathbb{K}$ for a $(k,\kappa)$-strongly controllable system. However, we highlight a crucial difference below.
 
Theorem~\ref{theorem:main} establishes that the infinite-horizon competitive control policy is contained in the class of DACs that utilize $\widehat{K}_0$ as the stabilizing baseline policy. Since this policy is derived from the competitive controller, which in turn depends on knowledge of $(A,B)$, such a $\widehat{K}_0$ might not be available when the underlying system is unknown, or even approximately known. To alleviate this concern, we give an analogous result that holds with respect to any baseline stabilizing policy $\mathbb{K}$, instead of the specific choice $\widehat{K}_0$.

\begin{theorem}[Generalized policy inclusion]\label{claim:main2}
Fix a horizon $T$. For any $\varepsilon > 0$, choose $$H = 2\log (10^5\beta^2W^2\kappa^{16} T^5 /(\gamma^4\varepsilon))/\gamma, \theta =20\kappa^5\beta^{1/2}/\gamma, \gamma' = \gamma/2$$ and define $\mathcal{M}$ be the set of $(H,\theta,\gamma')$-DAC policies with an arbitrary baseline stabilizing policy $\mathbb{K}$. Let $\mathcal{A}$ be an online algorithm with the optimal competitive ratio $\alpha^*$.
Then there exists a $\pi\in \mathcal{M}$ such that for any perturbation sequence $w_{1:T}$, 
 $$J_T(\pi|w_{1:T}) <  J_T(\mathcal{A}|w_{1:T}) + \varepsilon.$$
\end{theorem}

Now, with the aid of the above result, we may proceed as with Theorem~\ref{thm:main}. The algorithm and result in \cite{hazan2019nonstochastic} guarantees $\tilde{O}(T^{2/3})$ regret against $\K\cup \mathcal{M}$, and similarly those in \cite{simchowitzmaking} guarantee $\tilde{O}(\poly(\mu^{-1})\sqrt{T})$ regret -- say, this bound is $R_T$ in either case. Again, by  the non-negativity of the cost, and thereafter using the optimal competitive ratio bound, we have
\begin{align*}
  J_T(\mathcal{A} |w_{1:T}) &\leq  \min_{\pi\in \mathcal{M}} J_T(\pi | w_{1:T}) + R_T \\
  &\leq J_T(\mathcal{A}|w_{1:T}) +R_T + O(1)\\
  &\leq J_\infty(\mathcal{A}|w_{1:\infty}) + R_T + O(1)\\
  &\leq \alpha^*\cdot OPT_*(w_{1:\infty}) + R_T + O(1),
\end{align*}
where we extend $w_{1:T}$ to an infinite sequence $w_{1:\infty}$ with $w_t=0$ for all $t>T$.
An appeal to Theorem~\ref{mainmagic} concludes the claim.
\end{proof}

\begin{proof}[Proof of Theorem~\ref{thm:black}]
  Here, we use the algorithm and the regret bound from \cite{chen}, which guarantees $R_T = 2^{O_T(1)} +\tilde{O}_T( \min\{\sqrt{T},\poly(\mu^{1}\log T)\})$ regret again the combined policy class $\K\cup\mathcal{M}$. The rest of the proof is analogous to that of Theorem~\ref{thm:give} -- particularly in that we invoke Theorem~\ref{claim:main2}.
\end{proof}


\section{A generalized proof of policy inclusion}\label{sec:d}
\begin{proof}[Proof of Theorem \ref{claim:main2}] 
We begin by taking note of the explicit characterization of an optimal competitive-ratio algorithm presented in Theorem~\ref{crt}. In particular, unrolling the recursive definition of $\nu_t$, we have
$ \nu_t = \sum_{i\geq 1} (A-KQ^{1/2})^{i-1} w_{t-i}$. Let the {\em true} state sequence visited by such a competitive control policy be $\{x_t\}_{t=1}^T$. Additionally, let $\widehat{K} = \begin{bmatrix} \widehat{K}_0 & \widehat{K}_1 \end{bmatrix}$ be the partition of $\widehat{K}$ along the first $m$ columns and the remaining $m$. Since $\widehat{w}_t = \Sigma^{-1/2}Q^{1/2}\nu_t$, by the second part of the said theorem, we have 
\begin{align*}
    x_{t+1} &= Ax_t+B\widehat{K}\xi_t +w_t = Ax_t + B\widehat{K}\begin{bmatrix} x_t-\nu_t \\ \Sigma^{-1/2} Q^{1/2}\nu_t \end{bmatrix} + w_t \\
    & = (A+B\widehat{K}_0)x_t + B(\widehat{K_1} \Sigma^{-1/2}Q^{1/2} - \widehat{K}_0) \nu_t  + w_t\\
    &= \sum_{i\geq 0} (A+B\widehat{K}_0)^i \left(w_{t-i} + B(\widehat{K}_1\Sigma^{-1/2}Q^{1/2} -\widehat{K}_0) v_{t-i}\right)\\
    &= \sum_{i\geq 0} (A+B\widehat{K}_0)^i \left(w_{t-i} +  B(\widehat{K}_1\Sigma^{-1/2}Q^{1/2} -\widehat{K}_0) \sum_{j\geq 1} (A-KQ^{1/2})^{j-1} w_{t-i-j}\right)\\
    &= \sum_{i\geq 0} \left((A+B\widehat{K}_0)^i  + \sum_{j=1}^{i} (A+B\widehat{K}_0)^{i-j} B(\widehat{K}_1\Sigma^{-1/2}Q^{1/2} -\widehat{K}_0)  (A-KQ^{1/2})^{j-1}  \right) w_{t-i}
\end{align*}
We now choose a DAC policy intended to emulate the competitive control policy as
\begin{align*}
    M^{[i-1]} \defeq&  (\widehat{K}_1\Sigma^{-1/2}Q^{1/2}-\widehat{K}_0)(A-KQ^{1/2})^{i-1} \\
    &+ (\widehat{K}_0-\mathbb{K}) \left((A+B\widehat{K}_0)^{i-1}  + \sum_{j=1}^{i-1} (A+B\widehat{K}_0)^{i-j-1} B(\widehat{K}_1\Sigma^{-1/2}Q^{1/2} -\widehat{K}_0)  (A-KQ^{1/2})^{j-1}  \right)
\end{align*}
and define the associated action as $ u_t(M) \defeq \mathbb{K}x_t(M) + \sum_{i=1}^H M^{[i-1]} w_{t-i}$, where $x_t(M)$ is the state sequence achieved the execution of such a DAC policy.
\begin{align*}
    x_{t+1}(M) &= Ax_{t}(M) + B \left(\mathbb{K}x_{t}(M) + \sum_{i=1}^H M^{[i-1]} w_{t-i}\right) +  w_{t} \\
    &= \sum_{i\geq 0 } (A+B\mathbb{K})^{i} \left(w_{t-i} + B\sum_{j=1}^H M^{[j-1]} w_{t-i-j}\right)\\
    &= \sum_{i\geq 0}  \left((A+B\mathbb{K})^{i} + \sum_{j=1}^{\min\{H, i\}}  (A+B\mathbb{K})^{i-j} B M^{[j-1]} \right)w_{t-i}\\
\end{align*}
Before proceeding to establish that $x_t$ and $x_t(M)$ are close, let us notice that $M^{[i]}$'s decay. To this extent, recall that due to strong stability there exist matrices $\widehat{S},\widehat{L}, S,L$ with $\|\widehat{S}\|\|\widehat{S}^{-1}\|, \|S\|\|S^{-1}\|\leq \kappa$ and $\|L\|\leq 1-\gamma, \|\widehat{L}\|\leq 1-\gamma$ such that $\widehat{A}+\widehat{B}_u\widehat{K} = \widehat{S}\widehat{L}\widehat{S}^{-1}$ and $A-KQ^{1/2}=SLS^{-1}$. Noting that $\Sigma\succeq I$ by definition and that $\beta I\succeq Q,I$, we thus obtain
\begin{align*}
    \|M^{[i]}\| &\leq 2\kappa^2 \beta^{1/2} (1-\gamma)^{i} + 2 \kappa \left(\kappa (1-\gamma)^i + 2\beta^{1/2}\kappa^4 i(1-\gamma)^{i-1} \right)\\
    &\leq 4\kappa^2\beta^{1/2}(1-\gamma)^i + 16 \kappa^5\beta^{1/2} (1-\gamma/2)^i /\gamma\leq 20\kappa^5\beta^{1/2}(1-\gamma/2)^i/\gamma,
\end{align*}
where the last line follows from Lemma~\ref{lem:newlemon}.

Note the identity $X^n = Y^n + \sum_{i=1}^n Y^{i-1} (X-Y) X^{n-i}$. Using this twice, we have
\begin{align*}
    (A+B\widehat{K}_0)^i =& (A+B\mathbb{K})^i + \sum_{j=1}^i (A+B\mathbb{K})^{i-j}B(\widehat{K}_0-\mathbb{K})(A+B\widehat{K}_0)^{j-1},\\
    (A+B\widehat{K}_0)^{i-j} =& (A+B\mathbb{K})^{i-j} + \sum_{k=1}^{i-j} (A+B\mathbb{K})^{k-1} B(\widehat{K}_0-\mathbb{K})(A+B\widehat{K})^{i-j-k}.
\end{align*}
Thus proceeding with the above identities, this time invoking the strong stability of $\mathbb{K}$, we have
\begin{align*}
    \|x_{t} - x_{t}(M) \| \leq& \sum_{i\geq H+1}\sum_{j=H+1}^i   \left\|(A+B\mathbb{K})^{i-j}BM^{[j-1}]\right\|W \\
    \leq & \sum_{i\geq H+1} \frac{20\kappa^7\beta^{1/2}W}{\gamma} i(1-\gamma/2)^{i-1} \leq \frac{20\kappa^7\beta^{1/2}WT^2}{\gamma} (1-\gamma/2)^{H}
\end{align*}
Similarly the corresponding control inputs are nearly equal. By the previous display concerning states, we have
\begin{align*}
    u_t &= \widehat{K}_0 x_t + (\widehat{K}_1\Sigma^{-1/2}Q^{1/2} -\widehat{K}_0) \nu_t\\
    &= \widehat{K}_0 x_t + (\widehat{K}_1\Sigma^{-1/2}Q^{1/2} -\widehat{K}_0) \sum_{i\geq 1} (A-KQ^{1/2})^{i-1} w_{t-i},\\
    u_t(M) &= \mathbb{K}x_t(M) + \sum_{i=1}^H (\widehat{K}_1\Sigma^{-1/2}Q^{1/2} -\widehat{K}_0 )(A-KQ^{1/2})^{i-1} w_{t-i} + (\widehat{K}_0-\mathbb{K})x_t,\\
    \|u_t-u_t(M)\| &\leq \|\mathbb{K}\| \|x_t - x_t(M)\| + \sum_{i\geq H+1} \|(A-KQ^{1/2})\|^{i-1} 2\kappa \beta^{1/2} W \\
    &\leq \frac{20\kappa^8\beta^{1/2}WT^2}{\gamma} (1-\gamma/2)^{H} + 2\kappa^2 \beta^{1/2}W \sum_{i\geq H+1} (1-\gamma)^{i-1} \\
    &\leq \frac{22\kappa^8\beta^{1/2}WT^2}{\gamma} (1-\gamma/2)^{H}.
\end{align*}
 Being the iterates produced by a DAC policy, $x_t(M), u_t(M)$ are bounded as indicated by Lemma~\ref{lem:xudacboundnwq}. Now, in terms of cost, since $\|x\|^2 - \|y\|^2 = (x-y)^\top (x+y)$ and given Equation~\ref{eq:mbound}, we note that 
\begin{align*}
    &|x_t^\top Q x_t + u_t^\top Ru_t - x_t(M)^\top Q x^\top - u_t(M)^\top R u_t(M)|\\
    \leq & \beta\|u_t-u_t(M)\|(2\|u_t(M) \|+\|u_t-u_t(M)\|) + \beta\|x_t-x_t(M)\|(2\|x_t(M)\|+ \|x_t-x_t(M)\|)\\
    \leq & \frac{10^5 \beta^2 W^2 \kappa^{16} T^4}{\gamma^4} (1-\gamma/2)^H
\end{align*}
The difference in aggregate cost is at most $T$ times the quantity above. Therefore, setting $H=2\log (10^5\beta^2W^2\kappa^{16} T^5 /(\gamma^4\varepsilon))/\gamma$ is sufficient to ensure the validity of the claim. We have already established that the DAC satisfies $\theta = 20\kappa^5\beta^{1/2}/\gamma$ and $\gamma'=\gamma/2$.
\end{proof}

\section{To {\em finity} and beyond}\label{sec:e}
In this section, we relate the infinite-horizon cost $OPT_*(w_{1:\infty})$, as used in earlier proofs, to the finite variant $OPT_*(w_{1:T})$. By the non-negativity of the instantaneous cost,  it is implied that $OPT_*(w_{1:\infty}) \geq OPT_*(w_{1:T})$. We prove the reverse inequality, up to constant terms.
\begin{theorem}\label{mainmagic}
  Extend the perturbation sequence $w_{1:T}$ to an infinite sequence $w_{1:\infty}$, by defining $w_t=0$ for $t>T$. Then, we have
  $$ OPT_*(w_{1:\infty}) \leq OPT_*(w_{1:T}) + \poly(m,n,\kappa,\gamma^{-1},\mu^{-1},\beta)W^2.$$
\end{theorem}
\begin{proof}
Let $u_{1:T}^*(w_{1:T}) = \argmin_{u_{1:T}} J_T(u_{1:T}|w_{1:T})$; these are the actions of the finite-horizon optimal offline controller attaining cost $OPT_*(w_{1:T})$. Let $\mathcal{I}$ be an algorithm that executes $u_t^*(w_{1:T})$ for the first $T$ time steps and $u_t =  \widehat{K}_0 x_t$ (or any other $(\kappa,\gamma)$-strongly stable policy) thereafter. Let $x_t(\mathcal{I}), u_t(\mathcal{I})$ be the state-action sequence associated with $\mathcal{I}$. Now, we have by the definition of $OPT_*(w_{1:\infty})$, specifically its optimality, that 
\begin{align*}
 OPT_*(w_{1:\infty}) &\leq   J_\infty (\mathcal{I} | w_{1:\infty}) \\
 &= \sum_{t=1}^{T} c(x_t(\mathcal{I}), u_t(\mathcal{I})) + \sum_{t=T+1}^\infty c(x_t(\mathcal{I}), u_t(\mathcal{I}))\\
 &= OPT_*(w_{1:T}) + \sum_{t=T+1}^\infty c(x_t(\mathcal{I}), u_t(\mathcal{I}))
\end{align*}
where on the last line we use the observation that for the first $T$ steps $\mathcal{I}$ executes $u_t^*(w_{1:T})$. Henceforth, we wish to establish a bound on the remainder term. With this aspiration, let us note that $w_{t}=0$ for all $t>T$ implying for $t>0$ that
$$ x_{T+t}(\mathcal{I}) = (A+B\widehat{K}_0)^{t-1} x_{T+1}, $$ 
$$u_{T+t}(\mathcal{I}) = \widehat{K}_0(A+B\widehat{K}_0)^{t-1} x_{T+1}, $$
$$ \sum_{t=T+1}^\infty c(x_t(\mathcal{I}), u_t(\mathcal{I})) \leq \frac{2\beta \kappa^4 }{\gamma^2} \|x_{T+1}\|^2.$$

To conclude the proof of the claim, we invoke Lemma C.9 and C.6 in \cite{foster2020logarithmic}, whereby the authors establish a bound on the state sequence produced by an optimal offline policy using an explicit characterization of the same, to note that
$$ \|x_{T+1}\| \leq \poly(m,n,\kappa,\gamma^{-1}, \mu^{-1}, \beta)W.$$
\end{proof}

\section{Experiment Setup \& Further Evaluations}\label{sec:exp}
We compare 5 controllers in our experiments, the $H_2$ controller, the $H_{\infty}$ controller, the infinite horizon competitive controller from \cite{goel2021competitive}, the GPC controller from \cite{agarwal2019online} and the  ``clairvoyant offline'' controller which selects the optimal-in-hindsight sequence of controls. We consider two systems a standard 2 dimensional double integrator where the system is given as follows
\[ A = \begin{bmatrix} 1 & 1 \\ 0 & 1 \end{bmatrix}, B_u = \begin{bmatrix} 0 \\ 1 \end{bmatrix}, B_w = I, Q = I, R = I\]
and a system derived from Boeing \cite{varga2010detection}. The exact system details can be found in the code supplied with the supplementary. On a high level the system has a 5-dimensional state and a 9-dimensional control.

The perturbations we add fall into 6 categories.
\begin{itemize}
    \item Sin Perturbations: every entry of $[w_t]_i = \sin(\frac{8 \pi i}{T})$ where $T$ is the time horizon.
    \item Sin Perturbations with changing Amplitude: every entry of $[w_t]_i = \sin(\frac{8 \pi i}{T})\sin(\frac{6 \pi i}{T})$ where $T$ is the time horizon.
    \item Constant Perturbations: every entry of $w_t$ is set to 1.
    \item Uniform Perturbations: every entry of $w_t$ is drawn uniformly and independently in the range $[0,1]$.
    \item Gaussian Perturbations: $w_t$ drawn independently from $\mathcal{N}(0 , I)$.
    \item Gaussian Random Walk Perturbations: Every entry of $w_t$ drawn independently from $\mathcal{N}(w_{t-1} , I)$.
\end{itemize}

The performance of the controllers on the double integrator system can be found in 
Figure~\ref{fig:dint2} and the performance of the controllers on the Boeing system can be found in \ref{fig:more}. The experiments confirm our theorems, that over a large time horizon GPC performs at par and sometimes can be better than the competitive controller. We note that the optimal competitive ratio for the competitve controller turns out to be 14.67 in the double integrator system and 5.686 in the Boeing system.

\begin{figure}
     \centering
     \begin{subfigure}[b]{0.45\textwidth}
         \centering
         \includegraphics[width=\textwidth]{paper_plots/double_integrator/sin.png}
         \caption{Sinusoidal Perturbations}
         \label{fig:dint_sin}
     \end{subfigure}
     \hfill
     \begin{subfigure}[b]{0.45\textwidth}
         \centering
         \includegraphics[width=\textwidth]{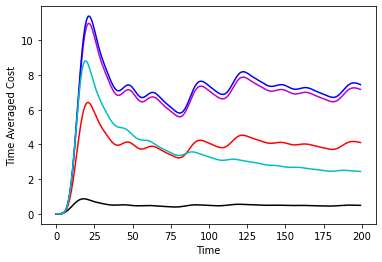}
         \caption{Constant Perturbations}
         \label{fig:dint_const}
     \end{subfigure}
     \bigskip 
     \begin{subfigure}[b]{0.45\textwidth}
         \centering
         \includegraphics[width=\textwidth]{paper_plots/double_integrator/const.png}
         \caption{Sin perturbations with changing amplitude}
         \label{fig:dint_sin}
     \end{subfigure}
     \hfill
     \begin{subfigure}[b]{0.45\textwidth}
         \centering
         \includegraphics[width=\textwidth]{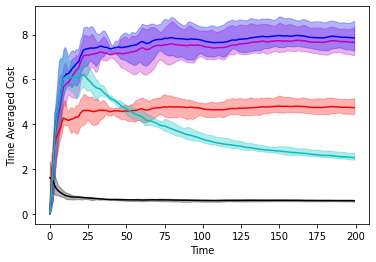}
         \caption{Uniformly distributed perturbations}
         \label{fig:dint_const}
     \end{subfigure}
     \bigskip 
     \begin{subfigure}[b]{0.45\textwidth}
         \centering
         \includegraphics[width=\textwidth]{paper_plots/double_integrator/gaussian.png}
         \caption{Gaussian Perturbations}
         \label{fig:dint_gaussian}
     \end{subfigure}
     \hfill
     \begin{subfigure}[b]{0.45\textwidth}
         \centering
         \includegraphics[width=\textwidth]{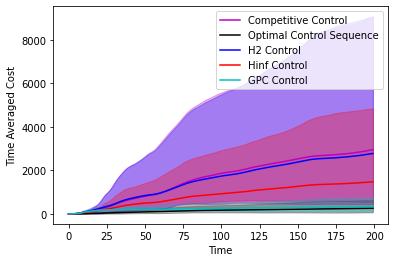}
         \caption{Gaussian Random Walk Perturbations}
         \label{fig:dint_grw}
     \end{subfigure}
     
        \caption{Performance of the various controllers on the Double Integrator system under different noise profiles Error ranges have been indicated over 5 trials for randomly generated noise sequences. In such case the solid lines indicate the mean performance over the 5 trials.}
        \label{fig:dint2}
        
\end{figure}

 \begin{figure}
      \centering
      \begin{subfigure}[b]{0.45\textwidth}
          \centering
          \includegraphics[width=\textwidth]{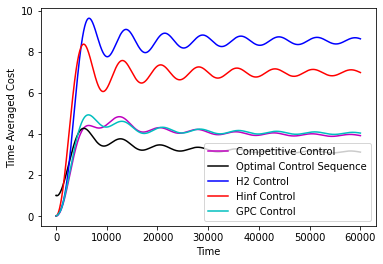}
          \caption{Sin perturbations}
          \label{fig:boeing_e1}
      \end{subfigure}
      \hfill
      \begin{subfigure}[b]{0.45\textwidth}
          \centering
          \includegraphics[width=\textwidth]{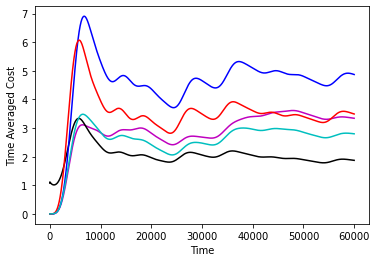}
          \caption{Sin perturbations with changing amplitude}
          \label{fig:boeing_e2}
      \end{subfigure}
      \bigskip 
      \begin{subfigure}[b]{0.45\textwidth}
          \centering
          \includegraphics[width=\textwidth]{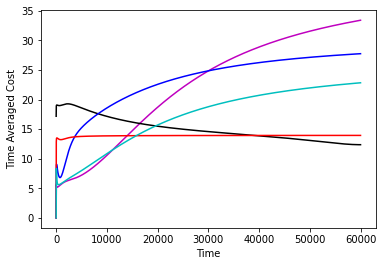}
          \caption{Constant perturbations}
          \label{fig:boeing_e3}
      \end{subfigure}
      \hfill
      \begin{subfigure}[b]{0.45\textwidth}
          \centering
          \includegraphics[width=\textwidth]{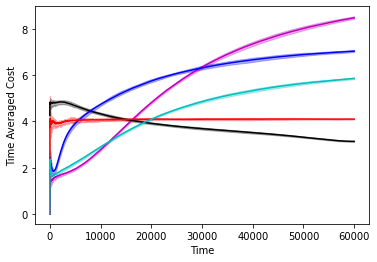}
          \caption{Uniformly distributed perturbations}
          \label{fig:boeing_e4}
      \end{subfigure}
      \bigskip 
      \begin{subfigure}[b]{0.45\textwidth}
          \centering
          \includegraphics[width=\textwidth]{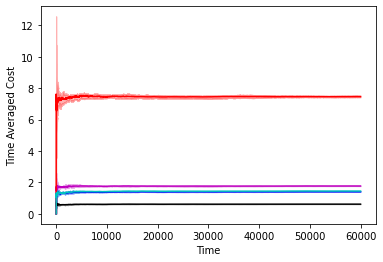}
          \caption{Gaussian perturbations}
          \label{fig:boeing_e5}
      \end{subfigure}
      \hfill
      \begin{subfigure}[b]{0.45\textwidth}
          \centering
          \includegraphics[width=\textwidth]{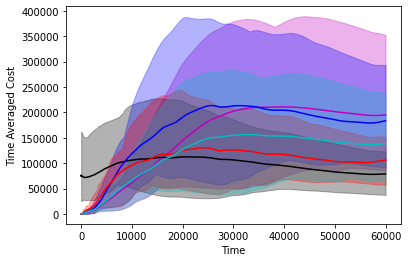}
          \caption{Gaussian Random Walk perturbations}
          \label{fig:boeing_e6}
      \end{subfigure}
      \hfill
         \caption{Performance of the various controllers on the Boeing system under different noise profiles Error ranges have been indicated over 5 trials for randomly generated noise sequences. In such case the solid lines indicate the mean performance over the 5 trials.}\label{fig:more}
 \end{figure}

\paragraph{Hyperparameter setup for GPC}

GPC is only the controller with hyperparameters. We describe the hyperparameters used here. 

\begin{itemize}
    \item Double Integrator System: For all perturbations $H = M = 3$. For all perturbations except the Gaussian Random walk perturbations, GPC was run with learning rate $0.002$. For Gaussian Random walk perturbations, GPC was run with learning rate $1e-4$. 
    \item Boeing System: For all perturbations $H = M = 10$. For all perturbations except the Gaussian Random walk perturbations, GPC was run with learning rate $0.001$. For Gaussian Random walk perturbations, GPC was run with learning rate $1e-7$.
\end{itemize}
We note that the learning rate needs to be reduced significantly for the Gaussian Random walk perturbations as the noise magnitude is growing in time and in this sense our theorem does not immediately apply to this setting.

\end{document}